\documentclass[twoside]{article}
\usepackage[accepted]{aistats2024}
\usepackage[round]{natbib}

\usepackage[utf8]{inputenc} 
\usepackage[T1]{fontenc}    
\usepackage{cvpr-abbriv}
\usepackage{url}            
\usepackage{booktabs}       
\usepackage{amsfonts}       
\usepackage{nicefrac}       
\usepackage{xcolor}         

\usepackage{amssymb}
\usepackage{amsmath}
\usepackage{amsthm}
\usepackage{stmaryrd}
\usepackage{mathtools}
\usepackage{bbm}
\usepackage{thmtools} 
\usepackage{thm-restate}
\usepackage{graphicx}
\usepackage{rotating}
\usepackage{hyperref}       
\definecolor{mydarkblue}{rgb}{0,0.08,0.45}
\hypersetup{%
  colorlinks=true,
  linkcolor=mydarkblue,
  citecolor=mydarkblue,
  filecolor=mydarkblue,
  urlcolor=mydarkblue
}

\usepackage[capitalise,nameinlink]{cleveref}
\theoremstyle{plain}

\newtheorem*{proposition*}{Proposition}

\theoremstyle{definition}

\theoremstyle{remark}

\newcommand{\koerper}[1]{\mathbb{#1}}
\newcommand{\R}{\koerper{R}}

\newcommand{\E}{\koerper{E}}
\newcommand{\prob}[1]{\mathbb{#1}}
\newcommand{\Ex}{\prob{E}}                                      

\providecommand{\cond}{\,\vert\,}                               
\DeclarePairedDelimiterX{\scalp}[2]{\langle}{\rangle}{#1,#2}    
\DeclarePairedDelimiterX{\infdivx}[2]{(}{)}{%
  #1\,\delimsize\|\,#2}
\newcommand{\kldiv}{D_{\rm KL}\infdivx}                             

\graphicspath{{.}}

\begin{document}
\twocolumn[
\aistatstitle{Symmetric Equilibrium Learning of VAE\lowercase{s}}
\aistatsauthor{ Boris Flach \And Dmitrij Schlesinger \And  Alexander Shekhovtsov }
\aistatsaddress{ Czech Techn.~University in Prague \And  Dresden University of Technology \And Czech Techn.~University in Prague } ]

\begin{abstract}
We view variational autoencoders (VAE) as decoder--encoder pairs, which map distributions in the data space to distributions in the latent space and vice versa.
The standard learning approach for VAEs is the maximisation of the evidence lower bound (ELBO). It is asymmetric in that it aims at learning a latent variable model while using the encoder as an auxiliary means only. Moreover, it requires a closed form a-priori latent distribution. This limits its applicability in more complex scenarios, such as general semi-supervised learning and employing complex generative models as priors. We propose a Nash equilibrium learning approach, which is symmetric with respect to the encoder and decoder and allows learning VAEs in situations where both the data and the latent distributions are accessible only by sampling. The flexibility and simplicity of this approach allows its application to a wide range of learning scenarios and downstream tasks. 
\end{abstract}

\section{INTRODUCTION}
Variational autoencoders \citep{Kingma:ICLR14,Rezende:ICML14} are a well established and well analysed approach of learning latent variable models of the form $p(x) = \sum_{z} p(z) p(x|z)$. Given a distribution $\pi(x)$, $x\in\mathcal{X}$ in the data space and an assumed distribution $p(z)$, $z\in\mathcal{Z}$ in the latent space, a VAE combines a pair of parametrised distributions $p_\theta(x\cond z)$, $q_\varphi(z\cond x)$, which are usually modelled in terms of deep networks. The standard way to learn this encoder--decoder pair is to maximise the evidence lower bound of the data log-likelihood, 
 \begin{align}
   L_B(\theta, \varphi) = \E_{\pi(x)} \bigl[& \E_{q_\varphi(z\cond x)} \log p_\theta(x \cond z) \label{eq:elbo-1} \\
   & - \kldiv[\big]{q_\varphi(z\cond x)}{p(z)}\bigr]  \nonumber .
  \end{align}
 This learning formulation is particularly well suited to situations where only the generative model $p(x)$ is of interest.
 The research in this area in recent years has culminated in deep hierarchical VAEs \citep{Vahdat:NeurIPS2020} and diffusion models \citep{Ho:NeurIPS2020,Rombach:CVPR2022}, which can be viewed also as hierarchical VAEs. 
 The encoder's role is auxiliary in the ELBO, and it is even fixed to a simple noisy shrinkage in diffusion models. However, a learned encoder is often of interest in applications on its own --- it can provide compact representations, useful for downstream tasks (\eg for semantic hashing,~\citealt{Dadaneh-20}).
 Furthermore, while only samples from $\pi(x)$ are needed in~\eqref{eq:elbo-1}, an explicit model of $p(z)$ is required in order to compute (and differentiate) the KL-divergence term. Although solutions to the latter problem have been proposed, they come with some other limitations (discussed in detail in \cref{sec:related}).

 The asymmetries of the standard VAE learning approach pointed above make it difficult to use it in semi-supervised training scenarios and in situations where both spaces $\mathcal{X}$ and $\mathcal{Z}$ are complex and possibly structured, as for instance in semantic segmentation with images $x$ and segmentations $z$. 
 Learning an encoder--decoder pair in such a scenario would naturally allow solving inference problems in both directions between $x$ and $z$ as well as to build more complex models. The requirement to model $p(z)$ by a simple and tractable density becomes then a significant limitation.
 
In this work, we propose a symmetric learning approach inspired by game theory, which leads to a simple learning algorithm. The method can handle implicitly given marginal distributions $\pi(x)$ and $\pi(z)$. It does not require gradients of parametric discrete expectations like the gradient of ELBO \wrt the encoder parameters, and therefore no reparametrisation is needed. Consequently, handling discrete or continuous variables is simple. 
The method gives a novel view of the well-known wake-sleep algorithm~\citep{Hinton:Science1995}, as discussed in~\cref{sec:related}.
It can be applied to models with structured latent spaces, like hierarchical VAE, and extended to models consisting of 3 or more groups of variables. In the latter case, the model consists of several inference networks -- one for each group of variables. They are learned jointly and can address an extended range of tasks at inference time, as we demonstrate experimentally.

The rest of the paper is organised as follows. In the next two sections we derive and analyse our novel learning approach. In the following section we exemplify its application to advanced models and learning setups. In the final experimental section we compare it with ELBO learning, show that it provides comparable model estimates, and demonstrate its applicability to more complex models not addressable by ELBO.

\section{PROBLEM FORMULATION} \label{sec:problem}
We propose a generic learning approach, whose primary goal is to learn a decoder $p(x\cond z)$ and an encoder $q(z\cond x)$ in the following training scenarios:
\par\noindent{\em Semi-supervised learning:} We assume training samples $x\sim\pi(x)$ and $z\sim\pi(z)$ and possibly also joint samples $(x,z)\sim\pi(x,z)$, i.i.d.~drawn from an unknown distribution $\pi(x,z)$ and its marginals.
\par\noindent{{\em Unsupervised learning:} Only samples of $x\sim\pi(x)$ are observed. In this case the space $\mathcal{Z}$ is a free modelling choice.

Similar to VAE learning, the choice of the models for the decoder and encoder is dictated by the need to be able to evaluate (or at least differentiate) their respective log-densities and to sample from them. We will assume that the decoder and encoder belong to parametric exponential families of the form 
\begin{subequations}\label{eq:cond1}
\begin{align}
    &p_\theta(x \cond z) \propto \exp\bigl[ \scalp{\phi(x)}{f_\theta(z)}],\\
    &q_\varphi( z \cond x) \propto \exp\bigl[ \scalp{\psi(z)}{g_\varphi(x)}],
\end{align}
\end{subequations}
where $\phi \colon \mathcal{X} \to \R^n$ and $\psi\colon \mathcal{Z} \to \R^m$ are fixed sufficient statistics. The mappings $f$ and $g$ are usually modelled by deep networks, parametrised by $\theta$, $\varphi$.  Notice that variables $x$, $z$ can be either discrete or continuous depending on the chosen exponential family. Common choices are \eg Bernoulli or Gaussian models. 
\section{SYMMETRIC EQUILIBRIUM LEARNING}\label{sec:eq-learning}
We present our general approach and theoretical analysis for the semi-supervised learning task from the previous section, which naturally calls for a symmetric formulation.

For simplicity of exposition, let us assume that only marginal empirical distributions $\pi(x)$ and $\pi(z)$ are given, but no joint observations are available.
The goal is to learn an encoder--decoder pair $q_\varphi(z\cond x)$ and $p_\theta(x\cond z)$ by (i) optimising the likelihood of the observed data and (ii) enforcing the encoder and decoder consistency at the same time. 
We formulate the learning task symmetrically as finding a {\em Nash equilibrium} for a two-player game. The strategy of the first player is represented by the decoder $p_\theta$. Similarly, the strategy of the second player is represented by the encoder $q_\varphi$. The utility function of a  player is the likelihood of the training data \wrt its strategy. Thereby, training examples are completed by the strategy of the other player. For example, the missing information in the examples $x\sim\pi(x)$ for the decoder likelihood is completed by the encoder strategy: $z \sim q_\varphi(z\cond x)$.
Proceeding in the same way for the encoder, we obtain the utility functions
\begin{subequations} \label{eq:utility-n}
  \begin{align}
    L_p(\theta, \varphi) &= \Ex_{\pi(x)} \Ex_{q_\varphi (z\cond x)}[\log p_\theta(x\cond z)] ,\label{eq:utility-n-p}\\
    L_q(\theta, \varphi) &= \Ex_{\pi(z)} \Ex_{p_\theta  (x\cond z)}[\log q_\varphi(z\cond x)] .
  \end{align}
\end{subequations}
As we will see later, the game aims at maximising the decoder likelihood and the encoder likelihood of the training data simultaneously, whereby the mutual completion reinforces decoder-encoder consistency.

A Nash equilibrium of the game is a pair $(\theta_*,\varphi_*)$ such that 
\begin{align}\label{Equilibrium-updates}
& L_p(\theta_*, \varphi_*) \geqslant L_p(\theta, \varphi_*), 
\hspace{.2em}\forall \theta , \nonumber \\
& L_q(\theta_*, \varphi_*) \geqslant L_q(\theta_*, \varphi), 
\hspace{.2em}\forall \varphi ,
\end{align}
\ie a point at which neither player can improve its objective function. Towards finding an equilibrium we consider a simple gradient algorithm, in which each player tries to improve its utility \wrt to its strategy 
\begin{align}\label{eq:nash-alg}
 \theta := \theta + \alpha \nabla_{\theta} 
 L_p(\theta, \varphi); \ \ \
 \varphi := \varphi + \alpha \nabla_{\varphi}
 L_q(\theta, \varphi).
\end{align}
These updates may be executed in parallel or sequentially. Stochastic unbiased estimates of the required gradients are readily obtained by differentiating Monte-Carlo estimates of expectations~\eqref{eq:utility-n} with as few as a single sample. 
Unlike in ELBO, the expectation $L_p(\theta, \varphi)$ does not need to be differentiated with respect to the encoder parameters and similarly for $L_q(\theta, \varphi)$. 
There is no need for the reparametrization trick in case of continuous variables or specialised gradient estimators through discrete samples in case of discrete variables.
\paragraph{Uniqueness}
It is well known that nonzero-sum games can have multiple and even infinitely many Nash equilibria. It is therefore crucial to analyse uniqueness of the solution as well as the convergence properties of the algorithm \eqref{eq:nash-alg}.

Extending the decoder and encoder to joint models via
\begin{align}\label{pq-joints}
    p_\theta(x,z) = p_{\theta}(x\cond z)\pi(z);\ \ q_\varphi(x,z) = q_{\varphi}(z\cond x)\pi(x)
\end{align}
the game utilities~\eqref{eq:utility-n} can be compactly written as
\begin{align}
    &\E_{q_\varphi(x,z)} \log p_\theta(x,z);
    &\E_{p_\theta(x,z)} \log q_\varphi(x,z).
\end{align}
This game is hard to analyse because of non-linear mappings involved.

To allow for theoretical analysis we will enlarge the spaces of feasible joint distributions by considering the following canonical exponential families
\begin{subequations}\label{eq:expf-lifted}
\begin{align}
    & p_u(x, z) = \pi(z) \exp\bigl[\scalp{\phi(x,z)}{u} - A(u)\bigr], \label{eq:expf_p}\\
    & q_v(x, z) = \pi(x) \exp\bigl[\scalp{\psi(x,z)}{v} - B(v)\bigr],
\end{align}
\end{subequations}
where $\phi(x,z), \psi(x,z)$ are sufficient statistics on $(x,z)$, $u$ and $v$ are free parameter vectors and $A$ and $B$ are cumulant functions ensuring normalisation. The models ~\eqref{eq:expf-lifted} are log-linear in $u$ and $v$ by design. At the same time, with sufficiently complex $\phi(x,z)$ and $\psi(x,z)$ they can represent or approximate all models from the original families which were parametrised in terms of neural networks.

We explain this model relaxation for the case of binary valued vectors $z$ and $x$. The components of the vector of natural parameters $f_\theta(z)$ in \eqref{eq:cond1} and the corresponding cumulant function are then pseudo-Boolean functions and can be written as polynomials in the components of $z$. The same holds for the components of the sufficient statistic vector $\phi(x)$. This means that if we take the components of $\phi(x,z)$ in the relaxed class to contain all base monomials, then for any $\theta$ there would be a corresponding parameter vector $u$ making the models equal. 
Notice that only under this correspondence the exponent part in~\eqref{eq:expf_p} matches the conditional distribution $p_\theta(x|z)$ while this is not true for a generic $u$. 

\begin{restatable}{theorem}{PropUnique}\label{T1}
 The two-player game with utility functions 
 \begin{subequations}\begin{align}
  & L_p(u,v) = \Ex_{q_v(x, z)} \log p_u(x, z), \\
  & L_q(u,v) = \Ex_{p_u(x, z)} \log q_v(x, z)
 \end{align}\end{subequations}
and strategies given by exponential family distributions \eqref{eq:expf-lifted}  
has a unique, asymptotically stable equilibrium.
\end{restatable}
The proof is given in Appendix~\ref{app:unique}. The idea is to construct a dual formulation of the game, which maximises the entropy under moment matching constraints. In this reformulation, it is then easy to prove the diagonal strict concavity condition~\citep{Rosen:Econ1965} -- a sufficient condition for uniqueness. Following theorems 7-10 in \citep{Rosen:Econ1965}, the theorem implies that the simple gradient ascent algorithm \eqref{eq:nash-alg} converges to the unique equilibrium point. 

The theorem applies to log-linear models~\eqref{eq:expf-lifted} with free natural parameters $u$ and $v$ and guarantees that the proposed algorithm converges to a unique equilibrium in this case. This has direct applicability to \eg EF-Harmonium models, which are however outside of our scope. Its value for VAEs defined in terms of neural networks is rather indirect: if the algorithm works in the lifted space, it gives more confidence that it would also make sense in a subspace with a non-linear parametrisation.

\paragraph{Consistency}
Finally, we discuss the question of encoder--decoder consistency. We say that models $p(x\cond z)$ and $q(z\cond x)$ are {\em consistent} if there exists a joint distribution $m(x,z)$ of which they are conditional distributions (see also~\citealt{liu2021on}).
Since we model $p_\theta(x \cond z)$ and $q_\varphi(z\cond x)$ independently, they are in general inconsistent. Enforcing the consistency strictly, while keeping the models in exponential families~\eqref{eq:cond1}, leads to a joint $m(x,z)$ necessarily collapsing to an EF-Harmonium~(\citealt{Arnold-1991}, \citealt{Shekhovtsov:ICLR2022}), which is a severe limitation. 
However, encouraging consistency could serve as a useful regularisation and can improve learning efficiency.

We observe that our game formulation implicitly encourages consistency.
\begin{restatable}{proposition}{PropConsistency}\label{P:consistency}
With the definition of joint distributions $p_\theta(x,z)$ and $q_\varphi(x,z)$ in~\eqref{pq-joints} and their respective marginals, the game \eqref{eq:utility-n} is equivalent to the game with utilities:
\begin{align*}
& \textstyle L'_p = \Ex_{\pi(x)}\big[ \log p_\theta(x) - \kldiv{q_\varphi(z\cond x)}{p_\theta(z\cond x)} \big],\label{equiv-ELBO} \\
& \textstyle L'_q = \Ex_{\pi(z)} \big[ \log q_\varphi(z) - \kldiv{p_\theta(x\cond z)}{q_\varphi(x\cond z)}\big] .
\end{align*}
\end{restatable}
See details in Appendix~\ref{app:unique}. 
The utility $L'_p$ is an alternative decomposition of ELBO into the data likelihood part and the encoder--posterior divergence, encouraging consistency. The utility $L'_q$ is a symmetric counterpart. The difference to ELBO learning is that $L'_p$ is optimised over $\theta$ only and not over $\varphi$ and vice-versa for $L'_q$.

Similar to ELBO learning, there is no guarantee that the proposed learning approach will result in a consistent decoder--encoder pair defining a unique joint distribution. The necessity for such a joint distribution might be however dictated by the application for which the VAE is learned. Or it might arise if the learned VAE is only a part of a larger model, which requires such a joint distribution. In such cases we may consider the distribution (\eg \citealt{liu2021on})
\begin{equation}\label{eq:implicit}
   \textstyle m(x,z) = \frac{1}{2} m(z) p(x\cond z) + \frac{1}{2} m(x) q(z\cond x)
\end{equation}
with implicitly defined marginals $m(x)$ and $m(z)$. They must satisfy $m(x) = \sum_z m(x,z)$ and $m(z) = \sum_x m(x,z)$, which leads to the equations 
\begin{subequations}\label{eq:m-margs}
\begin{align}
 m(x) & \textstyle = \sum_z p(x\cond z) m(z), \label{m(x)}\\
 m(z) & \textstyle = \sum_x q(z \cond x) m(x) .
\end{align}
\end{subequations}
While it is usually not possible to compute these marginals in closed form, it is nevertheless possible to sample from them and from the joint $m(x,z)$ as the limiting distributions of a Markov chain that alternates sampling of $x\sim p(x|z)$ and $z\sim q(z|x)$, as considered by \cite{NIPS2017_30f8f6b9}.
\section{ADVANCED MODELS AND LEARNING SETUPS}
In this section we exemplify the application of the proposed learning approach to several practically relevant learning setups and more complex models.

\paragraph{Semi-Supervised Learning with Mixed Data}
We extend the model and learning setup from ~\cref{sec:eq-learning} in two respects. First, we assume that in addition to empirical distributions $\pi(x)$ and $\pi(z)$ we also have complete training examples, \ie, matching pairs $(x,z)$, forming an empirical distribution $\pi(x,z)$. 
Note that here $\pi$-s are empirical distributions, hence \eg $\pi(x)$ need not be a marginal of $\pi(x,z)$. Second, we assume that the decoder's joint distribution is defined using its own parametrised prior for $z$, \ie $p_\theta(x,z)=p_\theta(z)p_\theta(x\cond z)$.

The utility function of the decoder sums the $p$-likelihoods of the training set, of which the likelihoods of examples $(x,z)\sim \pi(x,z)$ and $z\sim\pi(z)$, are tractable. 
The missing information in examples $x\sim\pi(x)$ with intractable $p$-likelihood is completed by the encoder strategy $q_\varphi(z\cond x)$.  Proceeding in the same way for the encoder, we get the utility functions
\begin{subequations} \label{eq:utility}
 \begin{align}
   L_p(\theta, \varphi) &= \Ex_{\pi(x,z)} [\log p_\theta(x,z)] + \Ex_{\pi(z)}[\log p_\theta(z)] + \nonumber \\
   & + \Ex_{\pi(x)} \Ex_{q_\varphi (z\cond x)}[\log p_\theta(x,z)] ,\\
   L_q(\theta, \varphi) &= \Ex_{\pi(x,z)} [ \log q_\varphi(z\cond x) ] + \nonumber \\
   & + \Ex_{\pi(z)} \Ex_{p_\theta(x\cond z)}[\log q_\varphi(z\cond x)] .
 \end{align}
\end{subequations}
 Although we follow the symmetric approach as before, the utilities \eqref{eq:utility} are not entirely symmetric due to the model asymmetry: $p_\theta(x, z)$ has its own parametrised prior $p_\theta(z)$, whereas $q_\varphi(z \cond x)$ lacks a prior model for $x$.

 \paragraph{Unsupervised Learning} By unsupervised learning we will understand the case when only $x\sim \pi(x)$ is observed.  The choice and interpretation of the $\mathcal{Z}$ space and the respective distribution is then completely free. We are interested in learning a decoder model $p_\theta(x,z) = p_\theta(x\cond z) p_\theta(z)$ and an encoder $q_\varphi(z \cond x)$ approximating $p_\theta(z \cond x)$.
 
 The utility function for the decoder is given by its likelihood for the examples $x\sim \pi(x)$, completed by the encoder. To form a likelihood for the encoder, we consider examples generated by the decoder model. The resulting utility functions are
 \begin{align} \label{eq:util_unsup}
  & L_p(\theta, \varphi) = \Ex_{\pi(x)} \Ex_{q_\varphi (z\cond x)}[\log p_\theta(x,z)], \nonumber \\
  & L_q(\theta, \varphi) = \Ex_{p_\theta(x,z)} [\log q_\varphi(z\cond x)].
 \end{align}
 In comparison with ELBO approach, the required stochastic gradients of the log-likelihoods are easy to compute, as discussed in~\cref{sec:eq-learning}.
 Notice that the algorithm applies also in case when $p(z)$ is fixed and implicit, \ie accessible by sampling only.

\paragraph{Hierarchical VAEs} Finally, we show that our unsupervised learning approach generalises to hierarchical / autoregressive VAEs. We assume that the hidden state $z$ consists of parts $z_0, z_1, \dots, z_m$, and $x\sim\pi(x)$ can be observed. Such models come in two variants. In the first one the factorisation order of the encoder is reverse to the factorisation order of the decoder. Examples are \eg Helmholtz machines \citep{Hinton:Science1995} and deep belief networks \citep{Hinton:NeurComp2006}. 
Here, we will consider the second variant, in which the encoder and decoder have the same order of factorisation:
\begin{subequations}
\begin{align}\label{HVAE-factorisation}
 & p(x,z) = p(z_0) \prod_{i=1}^m p(z_i \cond z_{<i}) \: p(x \cond z), \\
 & q(z \cond x) = q(z_0 \cond x) \prod_{i=1}^m q(z_i \cond z_{<i}, x) .
\end{align}
\end{subequations}
The encoder of such models can share parameters with the decoder, in particular \citet{Soenderby:NIPS2016} proposed to define the encoder by
\begin{equation}
 q_{\theta,\varphi}(z_i \cond z_{<i}, x) \propto p_\theta(z_i \cond z_{<i}) f_i(z_i; d_i(x, \varphi) ),
\end{equation}
where $f_i$ is a factorised function of $z_i$ and $d_i(x, \varphi)$ are the hidden layer outputs of a deterministic encoder network $x \mapsto d_{m} \mapsto d_{m-1}\dots \mapsto d_{0}$, parameterised by $\varphi$.
The strategy of the first player is represented by the decoder parameters $\theta$, while the strategy of the second player is represented by the encoder parameters $\varphi$. 
The utility functions for unsupervised learning are as in \eqref{eq:util_unsup}. Thanks to the factorisation of the decoder and encoder, they decompose into sums over the blocks $p(z_i \cond z_{<i})$ and $q(z_i \cond z_{<i}, x)$ and are tractable.

The model can be also learned ``partially'' semi-supervised by assuming that besides training examples $x\sim \pi(x)$ we have access to a (usually smaller) set of training examples $(x,z_0)\sim \pi(x,z_0)$. This is relevant, for example, when $z_0$ represents some hidden state(s) like classes or segmentations, on which we want to condition the decoder $p(x,z)$. The additional training examples will add
\begin{subequations}\label{semi-sup-labelled}
\begin{align}
  & \Ex_{\pi(x,z_0)}\Ex_{q(z_{>0} \cond z_0, x)} [\log p(x,z)], \\
  & \Ex_{\pi(x,z_0)}[\log q(z_0 \cond x)]
 \end{align}
\end{subequations} 
 to the respective utility functions.
\section{RELATED WORK}\label{sec:related}
\paragraph{Wake-Sleep}
The learning algorithm~\eqref{eq:nash-alg} with utility functions~\eqref{eq:util_unsup} in the unsupervised case turns out to be equivalent to the wake-sleep (WS) algorithm first proposed by \citet{Hinton:Science1995}. 
However, we arrived at it from a conceptually new game-theoretic formulation, allowing for new analysis and generalisation to other settings (semi-supervised, partial observation scenarios). In~\cref{A:WS} we give a brief overview of the original WS and follow-up works.

\paragraph{Implicit Prior}
An important advantage of the proposed method is allowing prior $\pi(z)$ to be implicit, \ie accessible via samples only. Several works have extended VAEs to handle implicit encoders and priors. \citet{Mescheder:ICML2017} and \citet{Huszar:Arxiv2017} estimate the log-density ratio $\log \frac{q_\varphi(z\cond x)}{\pi(z)}$ in ELBO by learning a logistic regression discriminator. Similar to GANs, this requires an inner loop with possibly complex discriminator. 
\citet{Molchanov-19} allow both the encoder and the prior to be an intractable mixture of tractable densities. At the training time, a finite sample from the mixture is used to form a density estimate of $\pi(z)$ and a lower bound on ELBO. These approaches are substantially more complex than ours and have further limitations. The prior can be made completely implicit, by assuming that the encoder-decoder model is consistent and hence defines a joint distribution and its marginals symmetrically. Towards this end \citet{liu2021on} explicitly optimise consistency and an expression that matches likelihood when assuming consistency.

\paragraph{Symmetric Learning}
Asymmetry of ELBO formulation has motivated several approaches, alternative to ours. \citet{Dumoulin-17} minimises Jensen-Shannon divergence between joint encoder $q(x,z) = \pi(x)q(x|z)$ and decoder $p(x,z)$. To estimate this divergence, a discriminator of joint samples is learned alongside, as in GANs. \citet{Pu-17} use a similar approach to minimise the symmetrised KL divergence. \citet{NIPS2017_30f8f6b9} learns the MCMC encoder--decoder sampler by using a discriminator between data-clamped and free-running chains. An important difference to our work is that the game in these approaches is between the discriminator and the model, not between decoder and encoder. 

\paragraph{Unsupervised and Semi-Supervised VAEs}
Unsupervised equilibrium learning with utilities~\eqref{eq:util_unsup} can be reinterpreted to facilitate theoretical comparison with ELBO alongside~\cref{P:consistency}.
Furthermore, hierarchical model with observed $z_0$~\eqref{semi-sup-labelled} is closely related to semi-supervised learning with ELBO~\citep{Kingma-14}.
These connections are detailed in~\cref{A:kingma}.

\section{EXPERIMENTS}\label{sec:exp}
\paragraph{Hierarchical VAE (MNIST)}
\begin{figure*}[t]
    \centering
    \begin{tabular}{ccc}
        & Random Latent Codes & Limiting Distribution \\[5pt]
        \begin{turn}{90} 
        \begin{minipage}{2.4cm} 
        \begin{center}
        ELBO
        \end{center}
        \end{minipage}
        \end{turn} & 
        \includegraphics[width=0.45\textwidth]{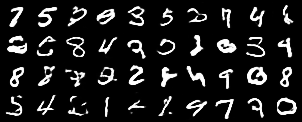} &
        \includegraphics[width=0.45\textwidth]{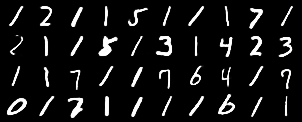} \\
        & $ \text{FID} = 5.17$ & $\text{FID} = 83.30$ \\[5pt]
        \begin{turn}{90} 
        \begin{minipage}{2.4cm}
        \begin{center}
        Symmetric
        \end{center}
        \end{minipage}
        \end{turn} & 
        \includegraphics[width=0.45\textwidth]{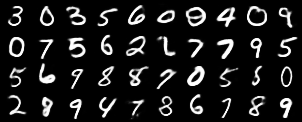} &
        \includegraphics[width=0.45\textwidth]{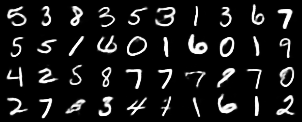} \\
        & $\text{FID} = 1.73$ & $\text{FID} = 3.63$ \\
    \end{tabular}
    \caption{\label{fig:mnist} Ladder VAE (MNIST): FID scores and images generated from random latent codes and from  limiting distributions of models learned by maximising ELBO and by symmetric equilibrium learning (images are shown by probabilities for better visibility).}
\end{figure*}
\begin{figure*}[t]
    \centering
    \includegraphics[width=0.98\textwidth]{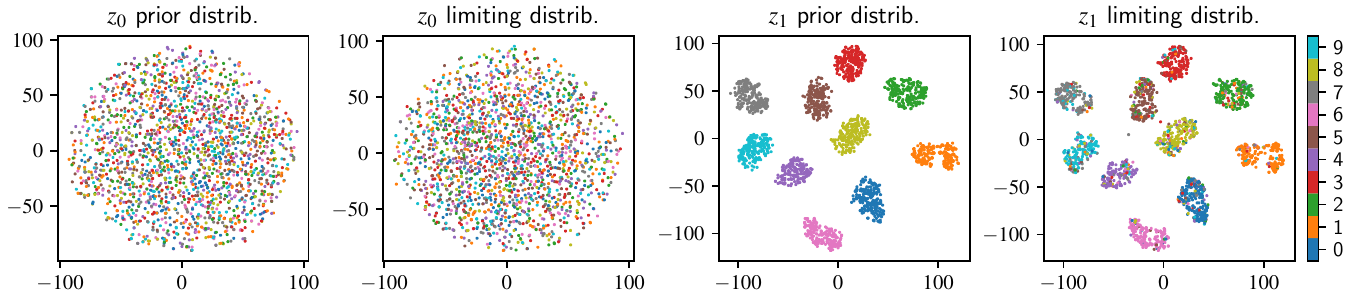}
    \caption{\label{fig:mnist-cls-tsne} MNIST: tSNE embeddings for the VAE with class labels. Points are coloured by digit classes. See text for explanation.}
   \end{figure*}
   
The goal of this experiment is to compare the symmetric equilibrium learning and ELBO learning on a simple dataset -- MNIST images binarised by a suitably chosen threshold. We consider  two hierarchical VAE model variants, each with two groups of binary valued latent variables $z_0\in\mathcal{B}^{30}$ and $z_1\in\mathcal{B}^{100}$. The decoder model is $p(x,z_0,z_1)=p(z_0) p(z_1 \cond z_0) p(x \cond z_1)$, where we assume a uniform distribution $p(z_0)$. The encoder for the first model variant (similar to ladder VAEs) factorises in the same order as the decoder, i.e.~$q(z_0, z_1 \cond x) = q(z_0\cond x) q(z_1\cond z_0, x)$ and shares parameters with the decoder as described in Sec.~\ref{sec:eq-learning}. The encoder in the second model variant factorises in reverse order, i.e.~$q(z_0, z_1 \cond x) = q(z_1\cond x) q(z_0\cond z_1)$ and shares no parameters with the decoder. The networks used in the encoders and decoders are standard deep convolutional networks of decreasing and increasing spatial resolution respectively. More details are provided in Appendix~\ref{app:mnist-impl}\footnote{The code is available under \\ \href{https://github.com/dschles70/symvae-aistats2024}{https://github.com/dschles70/symvae-aistats2024}}. Training such models with ELBO requires a specialised gradient estimator for differentiating expectations in $q$ w.r.t.~its parameters. We use the estimator by~\citet{Gregor-14}, which is superior to straight-through and comparable to complex unbiased estimators for VAEs~\citep{MuProp}. Notice again, that no such approximation is required for the symmetric equilibrium learning.

Besides validating the generative capabilities of two resulting hierarchical VAEs, we want to analyse the consistency of their decoder--encoder pairs. We therefore generate images (i) from the decoder model $p$ and (ii) from the limiting distribution $m(x)$ (see Sec.~\ref{sec:eq-learning} for explanation).
Fig.~\ref{fig:mnist} and Table~\ref{tab:mnist} indicate that the models obtained by symmetric learning achieves better consistency having at the same time slightly better FID scores. This is confirmed by tSNE embeddings of $z$ samples from the two models (see Appendix~\ref{app:mnist-impl}).

\begin{table}[h]
\caption{MNIST FID scores} \label{tab:mnist}
\begin{center}
\begin{tabular}{lll}
\textbf{model / alg.}  &\textbf{rand.~latent} &\textbf{limiting}\\
\hline \\
LVAE, ELBO & 5.17 & 83.30 \\
LVAE, symmetric & 1.73 & 3.63 \\
RVAE, ELBO & 5.83 & 29.59 \\
RVAE, symmetric & 0.81 & 5.40 \\
\end{tabular}
\end{center}
\end{table}

To further strengthen this finding, we conducted similar experiments for the Fashion-MNIST dataset. Results and details are given in Appendix~\ref{app:fmnist}.

The next experiment aims to show that the internal representations of a hierarchical VAE can be learned to have good generative and discriminative capabilities at the same time, even without ``supervised'' terms in the encoder objective as in \eqref{semi-sup-labelled}. For this we extend $z_0$ by ten additional binary variables, which encode the class labels (one hot encoding). This means that $z_0=(l,c)$ combines latent variables $l$ with class labels $c$. We learn the model by symmetric learning from labelled examples $(x,c)$, but use the following utility functions
\begin{align}\label{exp:semi-MNIST}
 & L_p(\theta, \varphi) = \Ex_{\pi(x,c)}\Ex_{q_\varphi(l \cond x)}\Ex_{q_\varphi(z_{>0} \cond x, z_0)} [\log p_\theta(x,z)] , \nonumber \\
 & L_q(\theta, \varphi) = \Ex_{p_\theta(x,z)} [\log q_\varphi(z \cond x)] .
\end{align}
This means that the class information is used only when learning the decoder (notice that $q_\varphi(c,l\cond x)$ factorises w.r.t.~to $c$ and $l$). The encoder is learned solely on examples generated from the decoder, \ie without any discriminative terms. The learned encoder achieves 99\% classification accuracy on the MNIST validation set, with almost no decrease of the FID scores for the generated images ($2.9$ when sampled from the decoder and $4.0$ when sampled from the limiting distribution). We also analyse tSNE embeddings of samples of the latent part $l$ of $z_0=(l,c)$ and samples of $z_1$, both from the prior distribution $p(z)$ and from the limiting distribution $m(z\cond c)$. Fig.~\ref{fig:mnist-cls-tsne} reveals that the latent part of $z_0$ is fully class agnostic, whereas $z_1$ is clearly clustered w.r.t.~the digit classes. This can be interpreted as follows. The latent part $l$ of $z_0=(l,c)$ is ``transversal'' to the class labels $c$ and presumably encodes image properties like stroke width, slant etc., whereas the internal representations in $z_1$ are clustered by digit classes and encode the appearance properties separately for each class.

\paragraph{Semantic Segmentation (CelebA)}
The following experiments illustrate the flexibility of the proposed approach on an application which is not accessible by ELBO learning. We consider the task of semantic segmentation with the goal to build a generative image segmentation model which can (i) generate image and segmentation pairs, (ii) segment given images, and (iii) generate images given a segmentation. 

We use the CelebA-HQ dataset \citep{karras2018progressive} and downscale its images and segmentations to $64\times 64$ pixels for simplicity.
\begin{figure*}[t]
    \centering
    \includegraphics[width=0.75\textwidth]{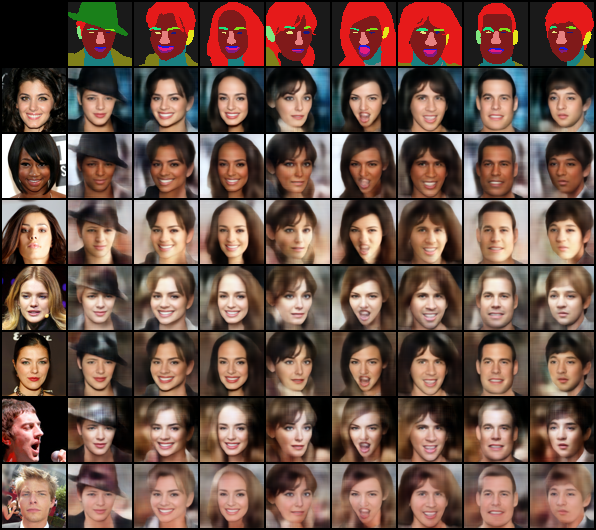}
    \caption{\label{fig:x_cond_s_panel} Given images and segmentations $(x_i,s_i)$ from the training set ($x_i$ are shown in the leftmost column), latent codes $z_{2i}$ are sampled from $q_{\varphi_2}(z_2 \cond x_i,s_i)$. Given segmentations $s_j$ shown in the top row, images $x_{i,j}$ are sampled from $p_{\theta_2}(x \cond s_j,z_{2i})$. Images are shown by mean values of the respective Gaussians for better visibility.}
\end{figure*}
Let $x\in\mathbb R^{3\times 64\times 64}$ be an image and $s\in \{1,\ldots,K\}^{64\times 64}$ be a segmentation (a categorical variable for each pixel). In order to model a distribution $p(x,s)$, we might try to learn a VAE with a decoder $p_\theta(x,s\cond z)$ and encoder $q_\varphi(z \cond x,s)$, assuming \eg a uniform prior distribution for the  vector of binary latent variables $z\in\mathcal{B}^m$.  However, this alone will not meet our goals because we can not access the resulting distributions $p(s \cond x)$ and $p(x \cond s)$. We propose to model $p_\theta(x,s\cond  z)$ as limiting distribution of a pair of parametrised conditional probability distributions $p_{\theta_1}(s\cond x,z)$ and  $p_{\theta_2}(x\cond s,z)$ (see \eqref{eq:implicit}). This means that the marginal probability distributions $p_\theta(x\cond z)$ and $p_\theta(s\cond z)$ are defined implicitly through the corresponding marginalisation constraints. 

To summarise, the whole model consists of three learnable conditional probability distributions $p_{\theta_1}(s\cond x,z)$, $p_{\theta_2}(x \cond s,z)$ and $q_\varphi(z \cond x,s)$. This defines a nested game with three players. Their respective strategies are represented by $\theta_1$, $\theta_2$ and $\varphi$. Their utility functions are
\begin{subequations} \label{eq:triple_game}
\begin{align}
 L_{\theta_1}(\theta_1, \theta_2, \varphi) & = 
 \Ex_{\pi(x,s)} \Ex_{q_\varphi(z\cond x,s)}
 \Bigl[ \log p_{\theta_1}(s \cond x, z) + \nonumber \\
 & \Ex_{p_{\theta_2}(x' \cond s, z)} \log p_{\theta_1}(s \cond x', z) \Bigr],
\end{align}
\begin{align}
 L_{\theta_2}(\theta_1, \theta_2, \varphi) & = 
 \Ex_{\pi(x,s)} \Ex_{q_\varphi(z\cond x,s)} 
 \Bigl[ \log p_{\theta_2}(x \cond s, z) + \nonumber \\
 & \Ex_{p_{\theta_1}(s' \cond x, z)} \log p_{\theta_2}(x \cond s', z) \Bigr],
\end{align}
\begin{align}
 L_{\varphi}(\theta_1, \theta_2, \varphi) & = 
 \Ex_{\pi(z)} \Ex_{p_\theta(x,s \cond z)} \Bigl[\log q_\varphi(z \cond x, s)\Bigr],
\end{align}
\end{subequations}
where Gibbs sampling is applied for obtaining pairs $(x,s)\sim p_\theta(x,s \cond z)$ in the last utility function. (See Appendix~\ref{app:segm-limiting} for detailed explanation).

To ease the training, we start by pre-training model parts for $p(s)$ and $p(x \cond s)$ separately. For the former we introduce latent variables $z_1\in \mathcal{B}^{50}$, which should encode segmentation shapes, and define $p(s)=\sum_{z_1} p(z_1)\cdot p_{\theta_1}(s \cond z_1)$ with uniform prior $p(z_1)$. The model for $p(x \cond s)$ is a latent variable model $p(x \cond s)=\sum_{z_2} p(z_2)\cdot p_{\theta_2}(x\cond s,z_2)$ with latent variables $z_2\in \mathcal{B}^{100}$, also uniformly distributed a-priori, which should encode appearance properties, like \eg segment colours, textures, characteristic shadows etc. Both $p_{\theta_1}(s \cond z_1)$ and $p_{\theta_2}(x\cond s,z_2)$ are equipped with corresponding encoders, \ie $q_{\varphi_1}(z_1\cond s)$ and $q_{\varphi_2}(z_2\cond x,s)$, and trained by symmetric learning, which is straightforward. All conditional probability distributions $p$ and $q$ are implemented as moderate complexity feed-forward networks, which output the parameters of the corresponding probability distribution. For example, $p_{\theta_2}(x\cond s,z_2)=\mathcal{N} (\mu_{\theta_2}(s,z_2), \sigma_{\theta_2}(s,z_2))$ is a diagonal normal distribution with means $\mu$ and standard deviations $\sigma$ provided by the corresponding network. 

Results for the learned $p_{\theta_2}(x\cond s)$ are illustrated in Fig.~\ref{fig:x_cond_s_panel} in the following way. We consider pairs of training examples, each consisting of an image and its segmentation. The first example is encoded by $q_{\varphi_2}(z_2\cond x,s)$ and the sampled latent code $z_2$ is used to decode the segmentation of the second example to an image by using $p_{\theta_2}(x\cond s,z_2)$.

After pre-training we extend the model part $p_{\theta_1}(s\cond z_1)$, learned in the previous step, to represent $p_{\theta_1}(s \cond x,z)$ by adding an ``additional branch'', \ie we define
\begin{align}\label{eq:new_codn_s}
    p(s\cond x,z) \propto \exp \bigl\langle f_1(z_1) + f_2(x,z_2), s_{oh}\bigr\rangle ,
\end{align}
where $f_1$ is the pre-trained network, $s_{oh}$ denotes the segmentation in one-hot encoding and $f_2$ is the additional network, which makes $s$ dependent on $x$ and $z_2$ as well. Its initial weights are chosen so that it outputs zeros at the beginning.

Finally, the model \eqref{eq:triple_game} is initialised by the pre-trained components and trained towards a Nash equilibrium for the three player game as explained above. 
Fig.~\ref{fig:res_full} shows a few results. The model achieves 95.2\% segmentation accuracy on the training set and 90.7\% segmentation accuracy on the validation set.

\begin{figure*}[t]
    \centering
    \includegraphics[width=0.75\textwidth]{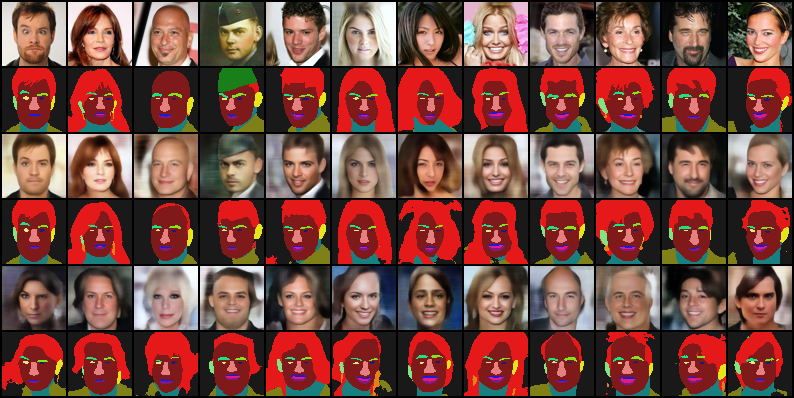}
    \caption{\label{fig:res_full} First two rows: training data $(x, s)$. Third and fourth rows: reconstructed images, and segmentations sampled from $p_\theta(x \cond s, z_2)$ and from $p_\theta(s\cond x, z)$ with $z\sim q_\varphi(z\cond x,s)$. Last two rows: sampling image--segmentation pairs from the full limiting distribution.}
\vskip0.5\baselineskip
    \centering
    \includegraphics[width=0.75\textwidth]{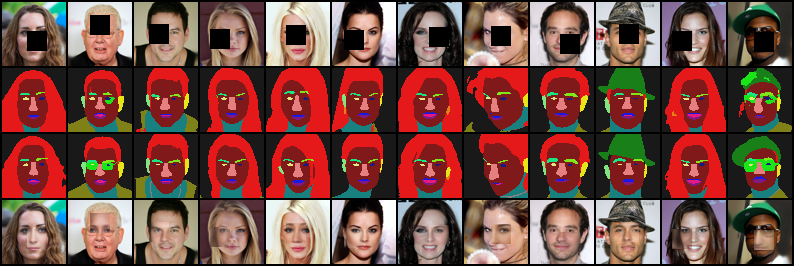}
    \caption{\label{fig:segminc} Segmentation from incomplete data. First row: original images from the validation set with hidden parts depicted as black squares. Second row: predicted segmentations, Third row: ground truth segmentations. Fourth row: ``in-painting'' -- average over all images obtained during generation (with clamped visible part).}
\end{figure*}

An important property of the obtained model is its ability to complete missing information for any subset of its variables. Given a partial observation -- \eg an image part, or a segmentation part, or a combination of such parts -- we can complete the missing data by sampling from the corresponding limiting distribution. We illustrate this property on the example of inference from incomplete images $x$. Let $x = (x_o, x_h)$ consist of two parts: an observed part $x_o$ and a hidden part $x_h$. In order to segment such an image by the maximum marginal decision strategy, we need to compute the marginal probabilities $p(s_i \cond x_o)$ for each pixel $i$. They can be estimated by Gibbs sampling, which alternately draws all unobserved random variables, including $x_h$. We accumulate segmentation label frequencies for each pixel during the sampling and finally decide for the label with the highest occurrence. A few results are presented in Fig.~\ref{fig:segminc}. As compared to the segmentation from complete images, the segmentation accuracy drops from 95.2\% to 92.8\% for the training set and from 90.7\% to 88.8\% for the validation set. We consider this accuracy drop as minor, because the segmentations inferred for the hidden image parts need not necessarily coincide with the ground truth -- they should only be ``plausible'', which is seen in the figure. Although not the primary goal of this experiment, Gibbs sampling allows at the same time to reconstruct the image content in the hidden parts (\ie in-painting). For this we employ a mean-marginal decision, \ie we average all sampled image values observed during Gibbs sampling. Although the results are sometimes not perfect (see the last row in Fig.~\ref{fig:segminc}), it is however enough to infer reasonable segmentations.

\section{CONCLUSION}
We propose an alternative learning approach for variational autoencoders. For this we view VAEs as decoder--encoder pairs and derive a symmetric learning formulation inspired by game theory, which leads to a simple learning algorithm for finding a Nash equilibrium. We prove its uniqueness under fairly general assumptions. The proposed method can be applied for various learning scenarios and for models with complex, possibly structured latent spaces. This includes implicit distributions in the latent space as well as discrete latent variables. We show experimentally that the models learned by this method are comparable to those obtained by ELBO learning and demonstrate its applicability for tasks that are not accessible by standard VAE learning.  
\subsubsection*{Acknowledgements}
We would like to thank our colleagues Tomas Werner and Denis Barucic for their continued interest in this work and their valuable comments and discussions which helped to improve the manuscript. We also thank the reviewers for their critical remarks, which encouraged us to present more experiments and to resolve remaining unclarities. 
B.F. gratefully acknowledges support by the Czech OP VVV project ”Research Center for Informatics” (CZ.02.1.01/0.0/0.0/16019/0000765).
D.S. was supported by the German Federal Ministry of Education and Research (BMBF) project 01/S18026A-F and by the German Federal Ministry for Economic Affairs and Climate Action (BMWK) project 01MN23021A.
A.S. was supported by the Czech Science Foundation grant GA24-12697S. The authors would like to thank the Center for Information Services and HPC (ZIH) at TU Dresden for providing computing resources.
\bibliography{strings,vae,sym-vae,ws}
\bibliographystyle{plainnat}
\clearpage
\appendix
\section{PROOFS} \label{app:unique}
In this section we provide proofs of formal claims regarding uniqueness and consistency-enforcement. 
\PropUnique*
\begin{proof}
We repeat here the model assumptions \eqref{eq:expf-lifted} for convenience
\begin{subequations} \label{eq:expf_pq}
 \begin{align}
  p_u(x, z) &= \pi(z) \exp\bigl[
 \scalp{\phi(x,z)}{u} - A(u)\bigr] \\
 q_v(x, z) &= \pi(x) \exp\bigl[
 \scalp{\psi(x,z)}{v} - B(v)].
 \end{align}
\end{subequations}
Our proof relies on the classic result of \citep{Rosen:Econ1965}, who shows that games satisfying {\em diagonal strict concavity} (DSC), a condition stronger than concavity, have unique Nash equilibria. 

Since the log-partition function of an exponential family is convex in its natural parameters, it follows that the game utilities are concave in their own strategies. A sufficient condition for the stronger DSC criterion is that the symmetrised Jacobian of the mapping
 \begin{equation}
  \begin{bmatrix}
   u \\
   v
  \end{bmatrix}
  \mapsto
  \begin{bmatrix}
   \nabla_u L_p(u, v) \\
   \nabla_v L_q(u, v)
  \end{bmatrix}
\end{equation}
is negative definite. The most convenient way to prove this condition is to ``dualise'' the game. Maximising $L_p(u,v)$ w.r.t.~$u$ is equivalent to finding the exponential family model, whose expected sufficient statistic $\Ex_{p_u}[\phi(x,z)]$  coincides with $\Ex_{q_v}[\phi(x,z)]$. This follows from
\begin{multline}
 \nabla_u \Ex_{q_v(x, z)}\log p_u(x,z) = \\ 
 \Ex_{q_v(x, z)}[\phi(x,z)] - \Ex_{p_u(x,z)}[\phi(x,z)].
\end{multline}
The corresponding dual task reads
\begin{subequations}\label{eq:dual-task}
\begin{align}
 & F_p(p) = \sum_{x,z} p(x,z) 
 \bigl[ \log p(x,z) - \log \pi(z) \bigr] 
 \rightarrow \min_p \\
 & \text{s.t.} 
  \left.\begin{cases*}
  \Ex_{p}[\phi(x,z)] = \Ex_{q}[\phi(x,z)]\\
  \sum_{x,z} p(x, z) = 1 .
  \end{cases*}\right.
\end{align}
\end{subequations}
This can be seen by noticing that \eqref{eq:dual-task} is a convex optimisation task with linear constraints. Hence, we can apply Fenchel duality
\begin{equation} \label{eq:fenchel}
 \inf_p\bigl\{F_p(p) \bigm| Ap = b \bigr\} = 
 \sup_\gamma\bigl\{\scalp{b}{\gamma} - F^*_p(A^T\gamma)\bigr\} ,
\end{equation}
where $F^*_p$ denotes the Fenchel conjugate function of $F_p$. For our case, we have $b=( \Ex_{q}[\phi(x,z)], 1)$ and the corresponding dual variables $\gamma = (u,\lambda)$. The Fenchel conjugate of the function $f(p) = p\log p - p \log \pi$ is $f^*(w) = \pi e^{w-1}$. Substituting all terms in the rhs of \eqref{eq:fenchel} and solving for $\lambda$, we get the task $\Ex_{q_v(x, z)} \log p_u(x, z)\rightarrow\max_u$. 

Applying the same dualisation for $L_q(u,v)$, we obtain the following ``dual''  game. The strategy of the first player is represented by $p(x, z)$ and the strategy of the second player is represented by $q(x,z)$. The utility functions $-F_p(p)$ and $-F_q(q)$ of the players depend on their respective strategy only. The game has additional linear constraints, where we assume existence of an interior feasible point $(p,q)$. 
The assertion of the theorem follows from Theorems 3,4,9 in \citep{Rosen:Econ1965}, if we prove that the symmetrised Jacobian of the mapping
 \begin{equation}
  \begin{bmatrix}
   p \\
   q
  \end{bmatrix}
  \mapsto
  \begin{bmatrix}
   \nabla_p F_p(p) \\
   \nabla_q F_q(q)
  \end{bmatrix}
\end{equation}
is positive definite. This is trivial since the Jacobian is diagonal with elements $1/p(x,z)$ in the first half of the diagonal and elements $1/q(x,z)$ in its second half. 
\end{proof}

\PropConsistency*
\begin{proof}
For completeness, we include the fact that $L'_p$ is an alternative form of the ELBO. It is verified as follows:
\begin{subequations}
\begin{align}
& \log p_\theta(x) - \kldiv{q_\varphi(z\cond x)}{p_\theta(z\cond x)} \\
& = \log p_\theta(x)  - \E_{q_\varphi(z\cond x)} \Big[\log \frac{q_\varphi(z\cond x)}{p_\theta(z\cond x)}\Big]\\
& = \E_{q_\varphi(z\cond x)} \Big[ \log p_\theta(x) - \log \frac{q_\varphi(z\cond x)}{p_\theta(z\cond x)}\Big]\\
& = \E_{q_\varphi(z\cond x)} \Big[ \log \frac{p_\theta(x) p_\theta(z\cond x)}{q_\varphi(z\cond x)}\Big]\\
& = \E_{q_\varphi(z\cond x)} \Big[ \log \frac{p_\theta(x|z) \pi(z)}{q_\varphi(z\cond x)}\Big]\\
& = \E_{q_\varphi(z\cond x)} \Big[ \log p_{\theta}(x|z) \Big] - \kldiv{q_\varphi(z\cond x)}{\pi(z)} \notag.
\end{align}
\end{subequations}
Therefore,
\begin{align}
    L_p' = L_p - \E_{\pi(x)} [\kldiv{q_\varphi(z\cond x)}{\pi(z)}].
\end{align}
Therefore, for a fixed $\varphi$, utilities $L_p'$ and $L_p$ share all local and global minima in $\theta$.
It is straightforward to see that $(\theta_*, \phi_*)$ is an equilibrium of the game with utilities $L_p'$ and  $L_q'$ iff it is an equilibrium of the game with utilities~\eqref{eq:utility-n}.
\end{proof}

\section{WAKE SLEEP}\label{A:WS}
In this section we give a brief overview of the original wake-sleep (WS) algorithm and follow-up works.

\citet{Hinton:Science1995} considered a multilayer network of stochastic neurons.
The “recognition” (encoder) connections are used to convert the input vector into a representation in one or more layers of hidden units.
The “generative” (decoder) connections are then used to reconstruct an approximation to the input vector from its underlying representation.
In the wake phase of WS, given an observed sample $x$ from the training dataset, a sample of hidden states $z$ is obtained from the encoder network and the decoder is learned on the joint sample $(x,z)$. In the sleep phase a joint sample is drawn from the decoder model and the encoder is learned. 

The model was initially assuming binary units and factorised encoder and decoder. In case of a hierarchical encoder--decoder model, the learning decouples over layers and no back-propagation is needed. Extended to a deep exponential family model~\citep{Vertes-18}, it is equivalent to a hierarchical VAE with the reverse encoder structure. 

\citet{Bornschein-15} et al.~uses importance sampling, similar to IWAE~\citep{Burda-15}, to tighten the bounds for the decoder and introduces a wake-phase (importance weighted) update of the encoder, tightening the ELBO (as in VAE) as well. 

\citet{Vertes-18} and \citet{Wenliang-20} showed that the encoder in WS can be specified implicitly by its mean parameters, which allows for non-conditionally independent encoders. This makes encoders more flexible so that higher quality decoder can be trained but impairs inference. 

The advantage of not requiring differentiation through discrete sampling has been explored by~\citet{Le-20} for models with stochastic control flow.

To our knowledge, prior work has neither extended WS to semi-supervised setting nor discussed the question of {\em why} it is a reasonable algorithm.
The only analysis attempt by~\citet{Ikeda-98} is limited to a strictly consistent encoder-decoder in a simple special case.

\section{(DIS-)SIMILARITIES TO ELBO}\label{A:kingma}
In this section we elaborate on similarities and difference between symmetric learning and ELBO learning in unsupervised as well as semi-supervised case \citep{Kingma-14}.

\paragraph{Unsupervised}
Recall, that in the unsupervised case we consider utility functions
\begin{align}\label{A:unsupervised-Lpq}
    & L_p(\theta, \varphi) = \Ex_{\pi(x)} \Ex_{q_\varphi (z\cond x)}[\log p_\theta(x,z)], \nonumber \\
    & L_q(\theta, \varphi) = \Ex_{p_\theta(x,z)} [\log q_\varphi(z\cond x)].
\end{align}
As discussed in~\cref{P:consistency}, the decoder utility can be equivalently replaced with the common ELBO $L_{B}(\theta, \phi)$ (both have the same dependence on $\theta$). The difference to VAE of~\citet{Kingma:ICLR14} is therefore only in the encoder learning. In VAE the encoder is learned to tighten ELBO, \ie to minimise the so-called {\em reverse} KL divergence in the expectation over the data distribution: 
\begin{align}
\textstyle \E_{\pi(x)}\big[ \kldiv{q_\varphi(z|x)}{p_\theta(z|x)} \big].
\end{align}
In the equilibrium learning, minimising $L_q$ in~\eqref{A:unsupervised-Lpq} \wrt encoder is equivalent to minimising
\begin{align}
\textstyle \E_{p_\theta(x)} \big[ \kldiv{p_\theta(z|x)}{q_\varphi(z|x)} \big],
\end{align}
which is a {\em forward} KL divergence between the same conditional distributions, and the expectation is over the generative model $p_\theta(x) = \sum_{z}p_{\theta}(z)p_{\theta}(x|z)$. 
The choice of the encoder as the true posterior, $q(z|x) = p(z|x)$, when possible (i.e.\ for consistent models), is optimal to both ELBO and symmetric learning. But in general, $L_q$ leads to different preferred solutions.

\paragraph{Semi-Supervised}
Semi-supervised learning of VAE was previously considered by~\citet{Kingma-14}. It can be seen that the hierarchical model~\eqref{HVAE-factorisation} is a generalisation of the generative model of~\citet{Kingma-14}: the state $z$ consists of two parts $(z_0,z_1)$, where $z_0$ is the image label, available only for a part of images. 
Similar to unsupervised case, when learning the decoder for a fixed encoder, the learning objective \cite[Eq. 8]{Kingma-14} is equivalent to our $L_p$. 

Only the learning of encoder differs. In their formulation the encoder minimises
\begin{align}
    & \E_{\pi(x)}\kldiv{q(z|x)}{p(z|x)} \\
    & + \E_{\pi(x,z_0)}\kldiv{q(z_{1}|x,z_0)}{p(x,z)} \notag \\
    & - \alpha \E_{\pi(x,z_0)} \log q(z_0|x) \notag,
\end{align}
where $\alpha$ is an empirical coefficient. In case when there are no unlabelled pairs, the first term disappears and the ELBO learning approach~\citep{Kingma-14} decouples into learning of a conditional VAE (decoder and encoder conditioned on $z_0$: $p(x|z_1,z_0)$, $q(z_1|x,z_0)$) and {\em an independent} discriminative learning of the encoder part $q(z_0|x)$ from the labelled data only. Thus, the generative counterpart of the model has no effect on learning of the recognition part (unless there is a parameter sharing).

In our formulation the encoder maximises
\begin{align}
 \E_{p(x,z)}\log q (z|x) + \E_{\pi(x,z_0)} \log q (z_0|x).
\end{align}
This objective is more homogeneous because both terms correspond to forward KL divergences. When there are no unlabelled training pairs, the objective stays the same and the encoder part $q(z_0|x)$ still needs to fulfil two goals: to approximate the posterior of the decoder $p(z_0|x)$ (in the expectation over the generated distribution $p(x)$, like in the unsupervised case) and to approximate the empirical distribution $\pi(z_0|x)$ (in the expectation over $\pi(x)$). A weighting coefficient might be appropriate here as well to balance the two objectives. 
Our semi-supervised MNIST experiment in~\cref{sec:exp} with utilities~\eqref{exp:semi-MNIST} shows that even when switching off the discriminative counterpart, the encoder still efficiently learns to classify.

\section{LEARNING MODELS WITH IMPLICIT MARGINALS} \label{app:segm-limiting}
Here we give a more detailed derivation of the learning in situations, where a joint model is given by means of its conditional distributions only, \ie marginal distributions are given implicitly. In particular, we used it in our experiments with CelebA to define and learn $p(x,s\cond z)$, where $x$ are images, $s$ are segmentations, and $z$ are latent variables. Since everything is conditioned on $z$ we will omit it for clarity and use $x$ and $s$ as variables of interest to be inline with our experiments. 

With the above agreement, we want to learn two conditional probability distributions $p_\theta(x\cond s)$ and $q_\varphi(s\cond x)$. As both images and segmentations are rather complex, it is desirable to avoid making any assumptions about the prior (marginal) distributions $p(s)$ and $q(x)$. Towards this end, we consider the MCMC process starting from a random state and alternately sampling using $p_\theta(x\cond s)$ and $q_\varphi(s\cond x)$. This process defines two limiting joint distributions, depending on which variable was sampled last:
\begin{equation}
m(s) p_\theta(x\cond s) \ \ \ \text{and} \ \ \ m(x) q_\varphi(s\cond x),
\end{equation}
where $m(x)$ and $m(s)$ are solutions to the stationary equations
\begin{subequations}\label{eq:appim_constr}\begin{align}
    m(x) &= \sum_s p_\theta(x\cond s)m(s) \\
    m(s) &= \sum_x q_\varphi(s\cond x)m(x).
\end{align}\end{subequations}
It is natural to consider the mixture of these two limiting distributions
\begin{equation}\label{eq:appim_impl}
    m(x,s) = \frac{1}{2} \Bigl[m(s) p_\theta(x\cond s) + m(x) q_\varphi(s\cond x)\Bigr] ,
\end{equation}
as we suggest in \eqref{eq:implicit}. Our goal therefore will be to maximise the likelihood of the data $\pi(x,s)$ under this mixture joint model. The likelihood can be lower-bounded w.r.t. mixture components as 
\begin{multline}\label{eq:appim_d1}
    \mathbb E_{\pi(x,s)}\log m(x, s) \\
    \geq \frac{1}{2}\Bigl[
        \mathbb E_{\pi(s)} \log m(s) + \mathbb E_{\pi(x,s)} \log p_\theta(x\cond s) + \\
    \mathbb E_{\pi(x)} \log m(x) + \mathbb E_{\pi(x,s)} \log q_\varphi(s\cond x) \Bigr] .
\end{multline}
Note that this lower bound is tight if the mixture components coincide, \ie $p_\theta(x\cond s)$ and $q_\varphi(s\cond x)$ are consistent.
The terms in \eqref{eq:appim_d1} corresponding to $p_\theta$ and $q_\varphi$ are tractable under assumption \eqref{eq:cond1}.
However, $m(x)$ and $m(s)$ are not given in closed form and depend on both $\theta$ and $\varphi$. We approximate their defining equations~\eqref{eq:appim_constr} as 
\begin{subequations}\label{eq:appim_constr-pi}\begin{align}
    m_\theta(x) &= \sum_s p_\theta(x\cond s)\pi(s) \\
    m_\varphi(s) &= \sum_x q_\varphi(s\cond x)\pi(x)
\end{align}\end{subequations}
and use these expressions in the mixture model~\eqref{eq:appim_impl}. With this approximation, \eqref{eq:appim_d1} sums the data likelihood terms with respect to separate model components $p_\theta(x\cond s)$, $m_\theta(x)$, $q_\varphi(s\cond x)$ and $m_\varphi(s)$. Hence, optimising this sum decouples into optimising the two objectives
\begin{align}
    L_p &= \mathbb E_{\pi(x,s)} \log p_\theta(x\cond s) + \mathbb E_{\pi(x)} \log m_\theta(x),\notag \\
    L_q & = \mathbb E_{\pi(x,s)} \log q_\varphi(s\cond x) + \mathbb E_{\pi(s)} \log m_\varphi(s)
\end{align}
independently in $\theta$ and $\varphi$, respectively. It remains only to explain how to handle $\log m_\theta(x)$ and $\log m_\varphi(s)$, which are still intractable. Substituting \eqref{eq:appim_constr-pi} and introducing a lower bound for $\log m_\theta(x)$ w.r.t.~summation over $s$ gives
\begin{multline}\label{eq:appim_d2}
    \mathbb  E_{\pi(x)} \log m_\theta(x) \geq 
     \mathbb E_{\pi(x)} \mathbb E_{q_\varphi(s\cond x)} \Bigl[\log p_\theta(x\cond s)
    + \\ \log \pi(s) - \log q_\varphi(s \cond x)\Bigr].
\end{multline}
If we consider the equilibrium learning approach, the objective $L_p$ is to be optimised only w.r.t.~its own parameters $\theta$, and therefore we can drop $\log \pi(s)$ and $\log q_\varphi(s \cond x)$ terms. Applying similar steps to $\Ex_{\pi(s)} \log m_\varphi(s)$ leads to the following effective equilibrium learning objectives:
\begin{multline}
    \tilde L_p(\theta, \varphi) = \Ex_{\pi(x,s)} \log p_\theta(x\cond s) + \\
     + \Ex_{\pi(x)}\Ex_{q_\varphi(s\cond x)} \log p_\theta(x\cond s),
\end{multline}
\begin{multline}
    \tilde L_q(\theta, \varphi) = \Ex_{\pi(x,s)} \log q_\varphi(s\cond x) + \\
     + \Ex_{\pi(s)}\Ex_{p_\theta(x\cond s)} \log q_\varphi(s\cond x) .
\end{multline}

Note that the first terms in these utilities correspond to the pseudo-likelihood objective, whereas the mutual completion in the second terms additionally enforces consistency.

\section{ADDITIONAL DETAILS FOR MNIST EXPERIMENTS} \label{app:mnist-impl}
\begin{figure}[t]
    \centering
    \includegraphics[width=0.8\linewidth]{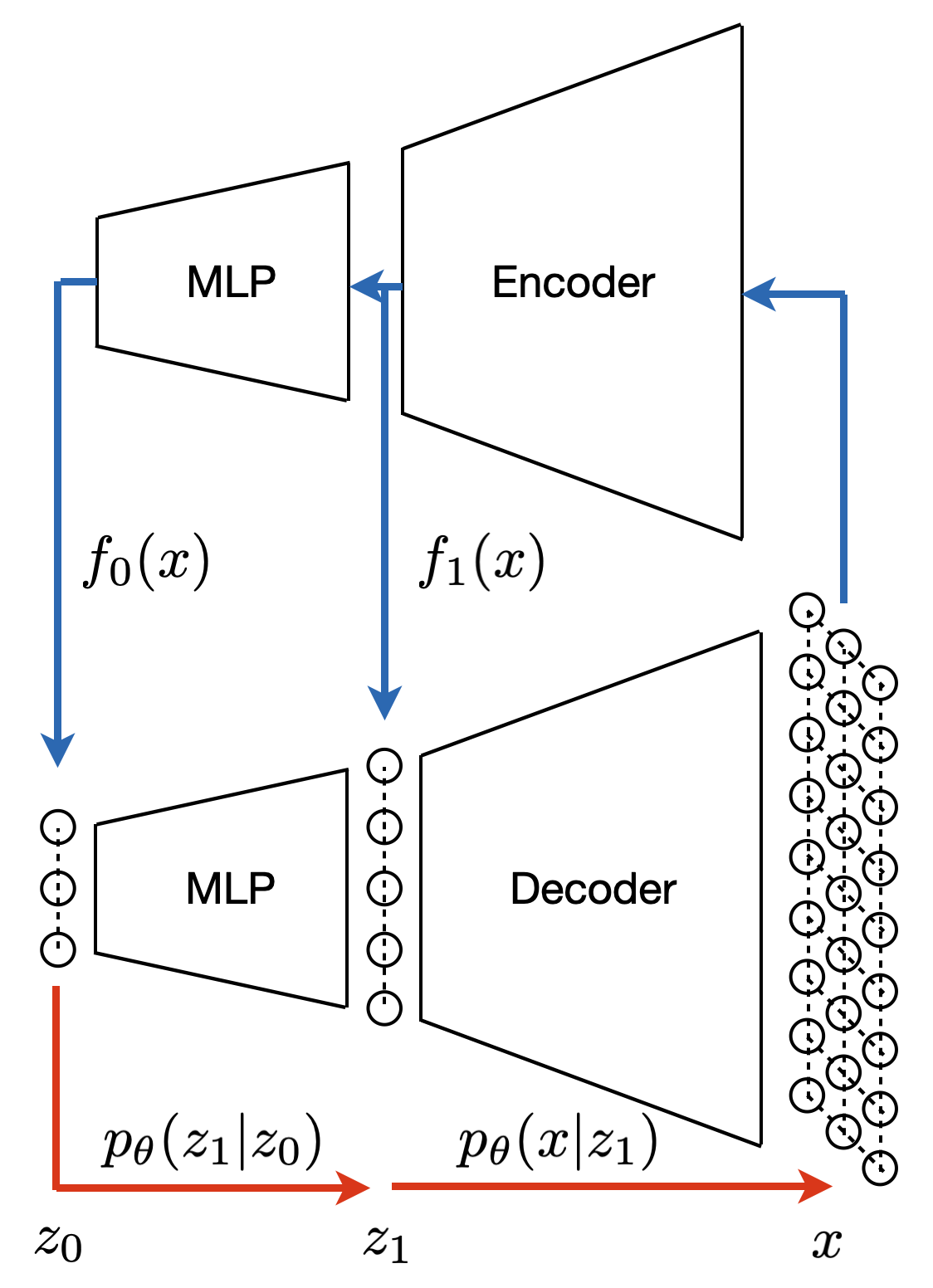}
    \caption{\label{fig:mnist-hvae} MNIST network architecture.}
\end{figure} 

\begin{figure}[t]
    \setlength{\tabcolsep}{0pt} 
    \begin{tabular}{cc}
    \includegraphics[width=\linewidth]{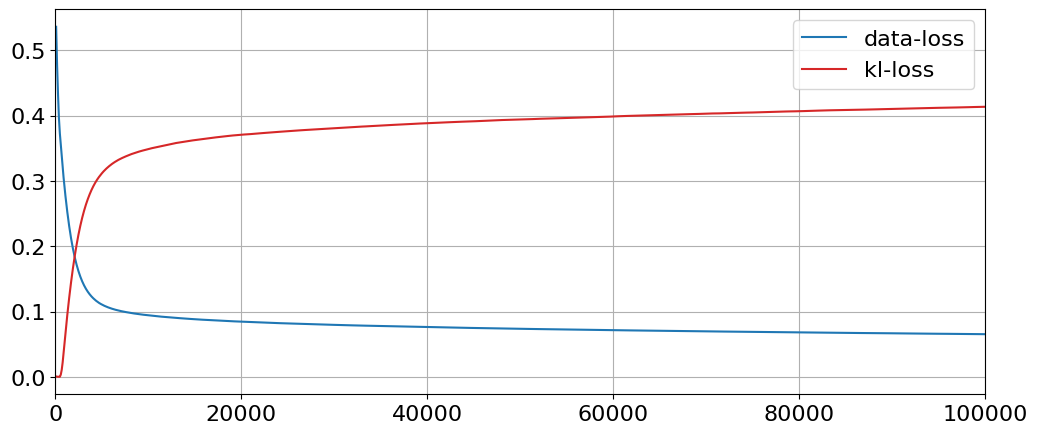}\\
    \includegraphics[width=\linewidth]{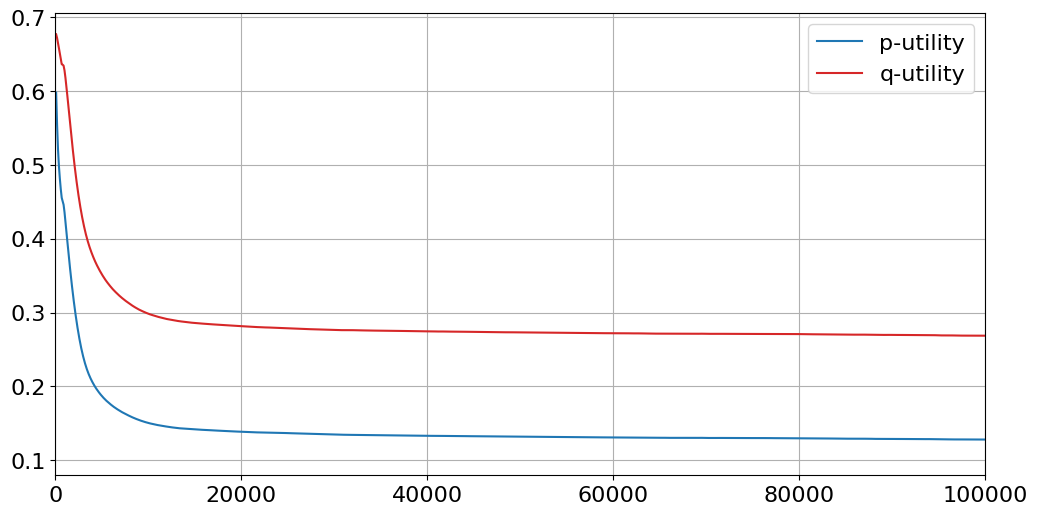}\\
    \end{tabular}
    \caption{\label{fig:mnist-hvae2} MNIST training with ELBO and symmetric learning. {\em Top}: data-term and KL-term for ELBO learning, {\em bottom}: negative utilities for symmetric learning.}
    \vskip0.5\baselineskip
\end{figure}
\begin{figure*}[t]
    \hspace{4ex}
    \includegraphics[width=0.45\textwidth]{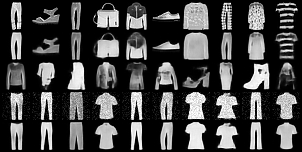}\hspace{2ex}
    \includegraphics[width=0.45\textwidth]{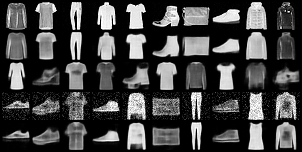}
    \caption{\label{fig:fmnist1} Fashion MNIST. Left: ELBO learning, right: symmetric learning. In each image: top row: original images, second row: means of the reconstructed images, i.e.~$\mu(z_1)$ with $z_1\sim q(z_1\cond x)$, third row: images generated from random codes (means visualised), fourth row: sampling from limiting distribution including image noise, last row: sampling from limiting distribution, means visualised.}
\vskip0.5\baselineskip
\begin{tabular}{ccc}
    & Random Latent Codes & Limiting Distribution \\[5pt]
    \begin{turn}{90} 
    \centering
    \ \ \ \ \ \ \ ELBO
    \end{turn} & 
    \includegraphics[width=0.45\textwidth]{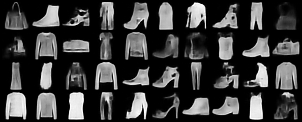} &
    \includegraphics[width=0.45\textwidth]{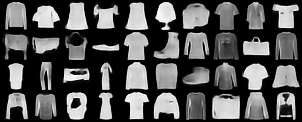} \\
    & $ \text{FID} = 32.37$ & $\text{FID} = 48.89$ \\[5pt]
    \begin{turn}{90} 
    \centering
    \ \ \ \ \ \ \ Symmetric
    \end{turn} & 
    \includegraphics[width=0.45\textwidth]{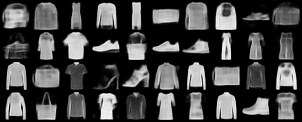} &
    \includegraphics[width=0.45\textwidth]{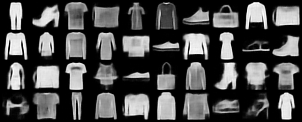} \\
    & $\text{FID} = 40.57$ & $\text{FID} = 37.85$ \\
\end{tabular}
\caption{\label{fig:fmnist2} Fashion MNIST. FID scores and images generated from random latent codes and from  limiting distributions of models learned by maximising ELBO and by symmetric equilibrium learning (images are shown by means for better visibility).}
\end{figure*}
Here, we provide additional implementation details for the HVAE models used by symmetric learning and by ELBO optimisation in the first MNIST experiment. The first model variant is defined by the decoder 
$\textstyle p_\theta(z_0,z_1,x)=p(z_0)p_\theta(z_1\cond z_0)p_\theta(x\cond z_1)$
and the encoder
$\textstyle  q_\varphi(z_0,z_1,x)=\pi(x)q_\varphi(z_0\cond x)q_\varphi(z_1\cond z_0, x),$
where $p(z_0)$ is uniform and $\pi(x)$ is the data distribution.
The network architecture is shown in \cref{fig:mnist-hvae}. The one-dimensional components are connected by a Multi-layer Perceptron (MLP) architecture. We used two hidden layers, $600$ hidden units each in our MLPs. Connections between $z_1$ and $x$ are implemented by standard convolutional encoder/decoder architectures with decreasing and increasing spatial resolutions respectively. Both encoder and decoder have 6 hidden layers, connected by 2D-convolution operations. In order to effectively reduce the spatial dimension some convolutions are performed with strides. We used the $\tanh$ activation function everywhere. The network weights are learned using the Adam-optimiser.

The hierarchical decoder consists of two ``separate'' networks, an MLP and a decoder, representing $p_\theta(z_1\cond z_0)$ and $p_\theta(x\cond z_1)$ respectively. The encoder corresponding to the direct factorisation order (shown in the figure) is a multi-head network. The common part is an encoder, which produces intermediate features, whereas the heads are an MLP for $f_0(x)$ and a single fully connected layer for $f_1(x)$. Two network outputs $f_0(x)$ and $f_1(x)$ serve as multiplier to the hierarchical decoder model, so $q_\varphi(z_0) = f_0(x)$ and $q_\varphi(z_1\cond z_0,x)\propto p_\theta(z_1\cond z_0)\cdot f_1(x)$. For the reverse factorisation order we keep the hierarchical encoder architecture basically the same but split it into two separate networks: the encoder for $q_\varphi(z_1\cond x)$ and the MLP for $q_\varphi(z_0\cond z_1)$.

The learning curves for losses/utilities are shown in \cref{fig:mnist-hvae2} for ELBO learning and symmetric learning respectively as a function of gradient update steps. For better clarity all values are normalised by the number of corresponding elements, e.g.~we show the per-pixel data-loss in ELBO. It is clearly seen that the convergence behaviours are pretty similar in both cases: all values converge very quickly to almost their final values, followed by a long period in which they change much more slowly. However, we observed that the quality of generated images keeps improving, even after the losses/utilities have almost reached saturation. Hence, we run all our experiments with a small learning rate of $10^{-4}$ for 1M gradient update steps (note: only first 100k steps are shown in \cref{fig:mnist-hvae2} for better visibility). 

We further compare the HVAE models obtained by symmetric learning and by ELBO optimisation by embedding samples for $z_0$ and $z_1$ from (i) the {\em prior} distributions $p(z_0)$, $p_\theta(z_1)$, (ii) the {\em posterior} distributions $q_\varphi(z_0)$, $q_\varphi(z_1)$, and (iii) the {\em limiting} distributions $m_{\theta,\varphi}(z_0)$ and $m_{\theta,\varphi}(z_1)$ for each of the two models by tSNE. \cref{fig:mnist-tsne} shows that all three samples match well for the model learned by symmetric learning. This is however not the case for the model learned by ELBO. 

\begin{figure*}[t]
    \centering
    \includegraphics[width=0.98\textwidth]{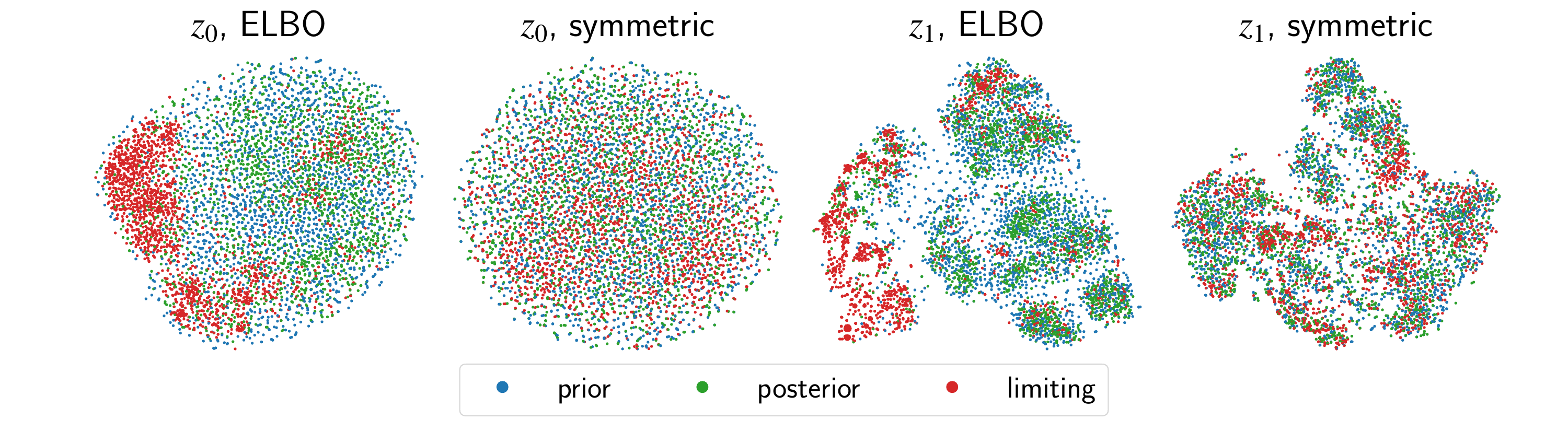}
    \caption{\label{fig:mnist-tsne} MNIST: tSNE embeddings of latent variables $z_0$ and $z_1$ for ELBO maximisation and symmetric learning.}
    \end{figure*}
\section{FASHION MNIST} \label{app:fmnist}
We also tested our approach for HVAE with the direct encoder factorisation order on the Fashion MNIST dataset. The model is exactly the same as the one used in our first MNIST experiment, except:
\begin{itemize}
    \item[--] Images are grey-valued now. We model them by a Gaussian, where the means for all pixels are computed by a network, and the standard deviation is common for all pixels and does not depend on $z$, i.e.~$p_{\theta,\sigma}(x\cond z_1)=\mathcal N(x;\mu_\theta(z_1),\sigma)$. The network architecture for $\mu_\theta(z_1)$ is the same as the decoder in the MNIST experiment, $\sigma$ is learned alongside with the network weights.
    \item[--] We observed that the overall results are slightly better (especially for ELBO), when using ReLU activations in $p(x\cond z_1)$ instead of $\tanh$ used for MNIST.
\end{itemize}
The results are shown in ~\cref{fig:fmnist1,fig:fmnist2}. They confirm our finding, that ELBO and symmetric learning are on par, whereby the latter produces more consistent encoder/decoder pairs.


\end{document}